\documentclass{article}
\usepackage{ijcai23}

\usepackage{times}
\usepackage{soul}
\usepackage{url}
\usepackage[hidelinks]{hyperref}
\usepackage[utf8]{inputenc}
\usepackage[small]{caption}
\usepackage{graphicx}
\usepackage{amsmath}
\usepackage{amsthm}
\usepackage{booktabs}
\usepackage{algorithm}
\usepackage{algorithmic}
\usepackage[switch]{lineno}

\usepackage{tikz,ifthen}
\usepackage{alltt}
\usepackage{paralist}
\usepackage{nicefrac}
\usepackage{xspace}
\usepackage{svg}
\usepackage{marvosym}
\usepackage{wasysym}
\usepackage{verbatim}
\usepackage{subfigure}
\usepackage{multirow}
\usepackage[capitalize]{cleveref}
\usepackage[per-mode=symbol,detect-all]{siunitx}
\usepackage{graphics}
\usepackage{amssymb}
\usepackage{wrapfig}
\usepackage{enumitem}
\usepackage{placeins}
\usepackage{caption}
\usepackage{xcolor}
\usepackage{url}
\usepackage{makecell}
\usepackage{chngcntr}
  
\usetikzlibrary{arrows,trees,backgrounds,automata,shapes,decorations,plotmarks,fit,calc,positioning,shadows,chains}

\definecolor{color0}{RGB}{239, 48, 84} 
\definecolor{color4}{RGB}{62, 255, 139} 
\definecolor{color6}{RGB}{147, 129, 255} 
\definecolor{color7}{RGB}{147, 129, 255} 

\definecolor{nnedgecolor}{RGB}{90,90,90}
\tikzstyle{every pin edge}=[<-,shorten <=1pt]
\tikzstyle{every path}=[draw=color7!50]
\tikzstyle{neuron}=[circle,fill=black!25,minimum size=17pt,inner sep=0pt]
\tikzstyle{input neuron}=[neuron, fill=color4]
\tikzstyle{output neuron}=[neuron, fill=color0]
\tikzstyle{hidden neuron}=[neuron, fill=color6]
\tikzstyle{annot} = [text width=4em, text centered]
\tikzstyle{nnedge} = [-{stealth},shorten >=0.1cm, shorten <=0.05cm,line 
width=0.8pt,nnedgecolor]
\usetikzlibrary{calc}

\newtheorem{definition}{\textbf{Definition}}

\newtheorem{theorem}{Theorem}

\title{The \#DNN-Verification Problem:\\ Counting Unsafe Inputs for Deep Neural Networks}

\author{
Luca Marzari*,
Davide Corsi*,
Ferdinando Cicalese\And
Alessandro Farinelli
\affiliations
Department of Computer Science, University of Verona, Verona, Italy
\emails
luca.marzari@univr.it, davide.corsi@univr.it
}

\begin{document}

\maketitle
\let\svthefootnote\thefootnote
	\let\thefootnote\relax\footnotetext{* Luca Marzari and Davide Corsi contributed equally.}
	\let\thefootnote\svthefootnote
	\addtocounter{footnote}{0}
\begin{abstract}
    Deep Neural Networks are increasingly adopted in critical tasks that require a high level of safety, e.g., autonomous driving.
    While state-of-the-art verifiers can be employed to check whether a DNN is unsafe w.r.t. some given property (i.e., whether there is at least one unsafe input configuration), their yes/no output is not informative enough for other purposes, such as shielding, model selection, or training improvements.
    In this paper, we introduce the \textit{\#DNN-Verification} problem, which involves counting the number of input configurations of a DNN that result in a violation of a particular safety property. We analyze the complexity of this problem and propose a novel approach that returns the exact count of violations. Due to the \#P-completeness of the problem, we also propose a randomized, approximate method that provides a provable probabilistic bound of the correct count while significantly reducing computational requirements. 
    We present experimental results on a set of safety-critical benchmarks that demonstrate the effectiveness of our approximate method and evaluate the tightness of the bound.
\end{abstract}

\section{Introduction}
In recent years, the success of Deep Neural Networks (DNNs) in a wide variety of fields (e.g., games playing, speech recognition, and image recognition) has led to the adoption of these systems also in safety-critical contexts, such as autonomous driving and robotics, where humans safety and expensive hardware can be involved. A popular and successful example is ``ACAS Xu", an airborne collision avoidance system for aircraft based on a DNN controller, which is now a standard DNN-verification benchmark. 

A DNN is a non-linear function that maps a point in the input space to an output that typically represents an action or a class (respectively, for control and classification tasks). A crucial aspect of these DNNs lies in the concept of generalization. A neural network is trained on a finite subset of the input space, and at the end of this process, we expect it to find a pattern that allows making decisions in contexts never seen before. However, even networks that empirically perform well on a large test set can react incorrectly to slight perturbations in their inputs \cite{adversarial}. While this is a negligible concern in simulated and controlled environments, it can potentially lead to catastrophic consequences for real safety-critical tasks, possibly endangering human health. Consequently, in recent years, part of the scientific communities devoted to machine learning and formal methods have joined efforts to develop DNN-Verification techniques that provide formal guarantees on the behavior of these systems \cite{Liu}\cite{ReluplexJournal}.

Given a DNN and a safety property, a DNN verification tool should ideally either ensure that the property is satisfied for all the possible input configurations or identify a specific example (e.g., adversarial configuration) that violates the requirements. 
Given these complex functions' non-linear and non-convex nature, verifying even simple properties is proved to be an NP-complete problem \cite{Reluplex}. In literature, several works try to solve the problem efficiently either by \textit{satisfiability modulo theories} (SMT) solvers \cite{Liu}\cite{Marabou} or by \textit{interval propagation} methods \cite{BetaCrown}. 

Although these methods show promising results, the current formulation, widely adopted for almost all the approaches, considers only the decision version of the formal verification problem, with the solution being a binary answer whose possible values are typically denoted \texttt{SAT} or \texttt{UNSAT}. \texttt{SAT} indicates that the verification framework found a specific input configuration, as a counterexample, that caused a violation of the requirements. \texttt{UNSAT}, in contrast, indicates that no such point exists, and then the safety property is formally verified in the whole input space. While an \texttt{UNSAT} answer does not require further investigations, a \texttt{SAT} result hides additional information and questions. For example, how many of such adversarial configurations exist in the input space? How likely are these misbehaviors to happen during a standard execution? Can we estimate the probability of running into one of these points?
These questions can be better dealt with in terms of the problem of counting {\em the number of violations} to a safety property, a problem that might be important also in other contexts:
 (i)\;\textit{model selection}: a counting result allows ranking a set of models to select the safest one. This model selection is impossible with a \texttt{SAT} or \texttt{UNSAT} type verifier, which does not provide any information to discriminate between two models which have both been found to violate the safety condition for at least one input configuration.
(ii)\;\textit{guide the training}: knowing the total count of violations for a particular safety property can help guiding the training of a deep neural network in a more informed fashion, for instance, minimizing this number over the training process.
(iii)\;\textit{estimating the probability of error}: the ratio of the total number of violations over the size of the input space provides an estimate of the probability of committing an unsafe action given a specific safety property.

Furthermore, by enumerating the violation points, it is possible to perform additional safety-oriented operations, such as:
(iv)\;\textit{shielding}: given the set of input configurations that violate a given safety property, we could adopt a more informative shielding mechanism that prevents unsafe actions.
(v)\;\textit{enrich the training phase}: if we can enumerate the violation configurations, we could add these configurations to the training (or to a memory buffer in deep reinforcement learning setup) to improve the training phase in a safe-oriented fashion. 
Motivated by the above questions and applications, previous works \cite{Baluta}\cite{zhang2021bdd4bnn}\cite{ghosh2021justicia} propose a \textit{quantitative} analysis of neural networks, focusing on a specific subcategory of these functions, i.e., Binarized Neural Networks (BNN). However, violation points are generally not preserved in the binarization of a DNN to a BNN nor conversely in the relaxation of a BNN to a DNN \cite{zhang2021bdd4bnn}. 

To this end, in this paper, we introduce the \textit{\#DNN-Verification} problem, which is the extension of the decision DNN-Verification problem to the corresponding counting version. Given a general deep neural network (with continuous values) and a safety property, the objective is to count the exact number of input configurations that violate the specific requirement. 
We analyze the complexity of this problem and propose two solution approaches. In particular, in the first part of the paper, we propose an algorithm that is guaranteed to find the {\em exact} number of unsafe input configurations, providing a detailed theoretical analysis of the algorithm's complexity. The high-level intuition behind our method is to recursively shrink the domain of the property, exploiting the \texttt{SAT} or \texttt{UNSAT} answer of a standard formal verifier to drive the expansion of a tree that tracks the generated subdomains. Interestingly, our method can rely on any formal verifier for the decision problem, taking advantage of all the improvements to state-of-the-art and, possibly, to novel frameworks.
As the analysis shows, our algorithm requires multiple invocations of the verification tool, resulting in significant overhead and becoming quickly unfeasible for real-world problems. For this reason, inspired by the work of \cite{SampleCount} on \textit{\#SAT}, we propose an approximation algorithm for \textit{\#DNN-Verification}, providing provable (probabilistic) bounds on the correctness of the estimation. 

In more detail, in this paper, we make the following contribution to the state-of-the-art:
\begin{itemize}
    \item We propose an exact count formal algorithm to solve \textit{\#DNN-verification}, that exploits state-of-the-art decision tools as backends for the computation.
    \item We present \texttt{CountingProVe}, a novel approximation algorithm for \textit{\#DNN-verification}, that provides a bounded confidence interval for the results.
    \item We evaluate our algorithms on a standard benchmark, ACAS Xu, showing that \texttt{CountingProVe} is scalable and effective also for real-world problems.
\end{itemize}
To the best of our knowledge, this is the first study to present \textit{\#DNN-verification}, the counting version of the decision problem of the formal verification for general neural networks without converting the DNN into a CNF.

\section{Preliminaries}\label{preliminaries}

\label{DNN}
Deep Neural Networks (DNNs) are processing systems that include a collection of connected units called neurons, which are organized into one or more layers of parameterized non-linear transformations. Given an input vector, the value of each next hidden node in the network is determined by computing a linear combination of node values from the previous layer and applying a non-linear function node-wise (i.e., the activation function). Hence, by propagating the initial input values through the subsequent layers of a DNN, we obtain either a label prediction (e.g., for image classification tasks) or a value representing the index of an action (e.g., for a decision-making task).
 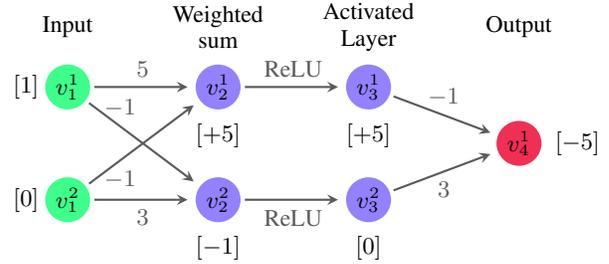
\begin{figure}[t]
	\begin{center}
		\scalebox{1.0} {
			\def\layersep{2.0cm}
			\begin{tikzpicture}[shorten >=1pt,->,draw=black!50, node
				distance=\layersep,font=\footnotesize]
				
				\node[input neuron] (I-1) at (0,-1) {$v^1_1$};
				\node[input neuron] (I-2) at (0,-2.5) {$v^2_1$};

				\node[left=-0.05cm of I-1] (b1) {$[1]$};
				\node[left=-0.05cm of I-2] (b2) {$[0]$};
				
				\node[hidden neuron] (H-1) at (\layersep,-1) {$v^1_2$};
				\node[hidden neuron] (H-2) at (\layersep,-2.5) {$v^2_2$};
				
				\node[hidden neuron] (H-3) at (2*\layersep,-1) {$v^1_3$};
				\node[hidden neuron] (H-4) at (2*\layersep,-2.5) {$v^2_3$};
				
				\node[output neuron] at (3*\layersep, -1.75) (O-1) {$v^1_4$};
				
				\draw[nnedge] (I-1) --node[above] {$5$} (H-1);
				\draw[nnedge] (I-1) --node[above, pos=0.3] {$-1$} (H-2);
				\draw[nnedge] (I-2) --node[below, pos=0.3] {$-1$} (H-1);
				\draw[nnedge] (I-2) --node[below] {$3$} (H-2);
				
				\draw[nnedge] (H-1) --node[above] {ReLU} (H-3);
				\draw[nnedge] (H-2) --node[below] {ReLU} (H-4);
				
				\draw[nnedge] (H-3) --node[above] {$-1$} (O-1);
				\draw[nnedge] (H-4) --node[below] {$3$} (O-1);

				\node[below=0.05cm of H-1] (b1) {$[+5]$};
				\node[below=0.05cm of H-2] (b2) {$[-1]$};

				\node[below=0.05cm of H-3] (b1) {$[+5]$};
				\node[below=0.05cm of H-4] (b2) {$[0]$};

                    \node[right=0.05cm of O-1] (b1) {$[-5]$};

				\node[annot,above of=H-1, node distance=0.8cm] (hl1) {Weighted
					sum};
				\node[annot,above of=H-3, node distance=0.8cm] (hl2) {Activated Layer };
				\node[annot,left of=hl1] {Input };
				\node[annot,right of=hl2] {Output };
			\end{tikzpicture}
		}
		\captionof{figure}{A simple example of a DNN $\mathcal{N}$ that will be used as a running example throughout the paper.}
		\label{fig:toyDnn}
	\end{center}
\end{figure}
Fig.~\ref{fig:toyDnn} shows a concrete example of how a DNN computes the output. Given an input vector $V_1=[1, 0]^T$, the weighted sum layer computes the value $V_2 = [+5, -1]^T$. This latter is the input for the so-called activation function \textit{Rectified Linear Unit} (ReLU), which computes the following transformation, $y = ReLU(x) = \max(0,x)$. Hence, the result of this activated layer is the vector $V_3=[+5, 0]^T$. Finally, the network's single output is again computed as a weighted sum, giving us the value $-5$.

\subsection{The DNN-Verification Problem}

An instance of the \textit{DNN-Verification} problem (in its standard decision form) is given by a trained DNN $\mathcal{N}$ together with a safety property, typically expressed as an input-output relationship for $\mathcal{N}$\;\cite{Liu}. In more detail, a property is a tuple that consists of a precondition, expressed by a predicate $\mathcal{P}$ on the input,  and a postcondition, expressed by a predicate $\mathcal{Q}$ on the output of $\mathcal{N}$. In particular, $\mathcal{P}$ defines the possible input values we are interested in---aka the input configurations---(i.e., the domain of the property), while $\mathcal{Q}$ represents the output results we aim to guarantee formally for at least one of the inputs that satisfy $\mathcal{P}.$ 
Then, the problem consists of verifying whether there exists an input configuration in $\mathcal{P}$ that, when fed to the DNN $\mathcal{N}$, produces an output satisfying $\mathcal{Q}$.  

\begin{definition}[\textit{DNN-Verification Problem}]
\label{def:decision_problem}
\phantom{a}

    {\bf Input}: A tuple $\mathcal{R}=\langle\mathcal{N}, \mathcal{P}, \mathcal{Q}\rangle$, where $\mathcal{N}$ is a trained DNN, $\mathcal{P}$ is precondition on the input, and $\mathcal{Q}$ a postcondition on the output.

    {\bf Output}: $\texttt{SAT}$ if 
    $\exists\;x\;|\;\mathcal{P}(x) \wedge \mathcal{Q}(\mathcal{N}(x))$ and $\texttt{UNSAT}$ otherwise,
    indicating that no such $x$ exists.
\end{definition}

As an example of how this problem can be employed for checking the existence of unsafe input configurations for a DNN,  
suppose we aim to verify that the DNN $\mathcal{N}$ of Fig.~\ref{fig:toyDnn}, for any input in the interval $[0,1]$, outputs a value greater than or equal 0. Hence, we define 
$\mathcal{P}$ as the predicate  on the input vector ${\bf v} = (v^1_1, v^1_2)$ which is true  iff  ${\bf v} \in [0,1]\times [0,1]$, and $\mathcal{Q}$ as the predicate on the output $v_4^1$ which is 
true iff $v_4^1 = \mathcal{N}(v^1_1, v^1_2) < 0,$
that is we set
$\mathcal{Q}$ to be the negation of our desired property.
Then, solving the \textit{DNN-Verification Problem} on the instance $(\mathcal{N}, \mathcal{P}, \mathcal{Q})$ we get 
$\texttt{SAT}$ iff there is counterexample that violates our property.

Since for the input vector ${\bf v} = (1, 0)$ (also reported in Fig.~\ref{fig:toyDnn}), the output of $\mathcal{N}$ is $<0$, 
in the example, the result of the \textit{DNN-Verification Problem} (with the postcondition being the negated of the desired property) is $\texttt{SAT}$, meaning that there exists at least a single input configuration $(v^1_1, v^1_2)$ that satisfies $\mathcal{P}$ and for which $\mathcal{N}(v^1_1, v^1_2) < 0$. 
As a result, we can say that the network is not safe for the desired property.\footnote{ More details about the problem and state-of-the-art methods for solving it can be found in the supplementary material.}
\section{\#DNN-Verification and Exact Count}
\label{exactCount}

In this section, we first provide a formal definition for \textit{\#DNN-Verification}; hence, we propose an algorithm to solve the problem by exploiting any existing formal verification tool.  Finally, we proceed with a theoretical discussion on the complexity of the problem, highlighting the need for approximate approaches. 

\subsection{Problem Formulation}

Given a tuple $\mathcal{R}=\langle\mathcal{N}, \mathcal{P}, \mathcal{Q}\rangle$, as in Definition~\ref{def:decision_problem}, 
we let $\Gamma(\mathcal{R})$ denote the set of  \textit{all} the input configurations for  $\mathcal{N}$ satisfying the property defined by $\mathcal{P}$ and $\mathcal{Q},$
i.e.
\[ \Gamma(\mathcal{R}) = \Bigg\{ x \; \big\vert \; \mathcal{P}(x) \wedge \mathcal{Q}(\mathcal{N}(x)) \Bigg\} \]
Then, the 
\textit{\#DNN-Verification} consists of computing the cardinality of
$\Gamma(\mathcal{R}).$

\begin{definition}[\textit{\#DNN-Verification Problem}]
\phantom{a}

{\bf Input}: A tuple $\mathcal{R}=\langle\mathcal{N}, \mathcal{P}, \mathcal{Q}\rangle$, as in Definition~\ref{def:decision_problem}. 

{\bf Output}: 
$\vert \Gamma(\mathcal{R})\vert$
\end{definition}


For the purposes discussed in  the introduction, rather than the cardinality of $\Gamma(\mathcal{R})$, it is more useful to define the problem in terms of the 
ratio between the cardinality of $\Gamma$ and the cardinality of the set of inputs satisfying $\mathcal{P}.$ 
%
We refer to this ratio as the \textit{violation rate (VR)}, and study the result of the \#DNN-Verification problem in terms of this equivalent measure. 

\begin{definition}[\textit{Violation Rate (VR)}]\label{violation}
    Given an instance of the \textit{DNN-Verification}  problem $\mathcal{R}=\langle\mathcal{N}, \mathcal{P}, \mathcal{Q}\rangle$
    we define the violation rate as 
    \[ VR = \frac{\vert\Gamma(\mathcal{R})\vert}{\lvert\{ x\;|\;\mathcal{P}(x) \}\lvert} \]
\end{definition}

Although, in general, DNNs can handle continuous spaces, in the following sections (and for the analysis of the algorithms), without loss of generality, we assume the input space to be discrete. We remark that for all practical purposes, this is not a limitation since we can assume that the discretization is made to the maximum resolution achievable with the number of decimal digits a machine can represent. It is crucial to point out that discretization is not a requirement of the approaches proposed in this work. In fact, supposing to have a backend that can deal with continuous values, our solutions would not require such discretization.

\subsection{Exact Count Algorithm for \#DNN-Verification}
We now present an algorithm to solve the exact count of \textit{\#DNN-Verification}. 

The algorithm recursively splits the input space into two parts of equal size as long as it contains both a point that violates the property (i.e.,  $(\mathcal{N}, \mathcal{P}, \mathcal{Q})$ is a $\texttt{SAT}$-instance for \textit{DNN-Verification} problem) and a point that satisfies it (i.e.,  $(\mathcal{N}, \mathcal{P}, \neg\mathcal{Q}$) is a $\texttt{SAT}$-instance for \textit{DNN-Verification} problem).\footnote{Any state-of-the-art sound and complete verifier for the decision problem can be used to solve these instances. In fact, our method works with any state-of-the-art verifiers, although, using a verifier that is not complete can lead to over-approximation in the final count.}  
The leaves of the recursion tree of this procedure correspond to a partition of the input space into parts 
where the violation rate is either 0 or 1. Therefore, 
the overall violation rate is easily computable by
summing up the sizes of the subinput spaces in the leaves of violation rate 1.
\begin{figure}[t]
    \centering
    \includegraphics[width=\linewidth]{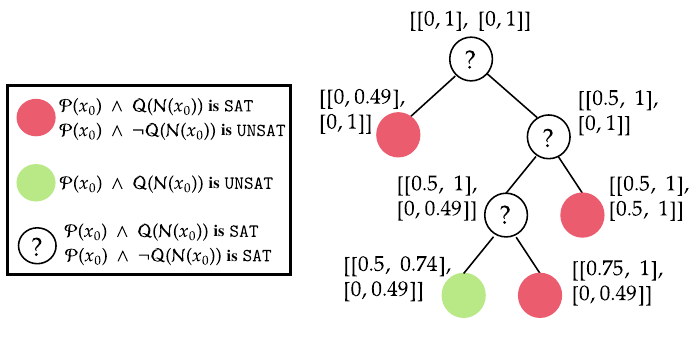}
    \caption{Example execution of exact count for a particular $\mathcal{N}$ and safety property (assuming a discretization factor of $0.01$).}
    \label{fig:exactCount}
\end{figure}
Fig.~\ref{fig:exactCount} shows an example of the execution of our algorithm for the DNN and the safety property presented in Sec.~\ref{preliminaries}. We want to enumerate the total number of input configurations ${\bf v} = (v_1^1, v_1^2)$ where both $v_1^1$ and $v_1^2$ satisfy $\mathcal{P}$ (i.e., they lie in the interval $[0,1]$) and such that the output $v^1_4 = \mathcal{N}({\bf v})$ is a value strictly less than 0, i.e., ${\bf v}$ violates the safety property. 

The algorithm starts with the whole input space and checks if at least one point exists that outputs a value strictly less than zero. 
The exact count method checks the predicate $\exists x \mid \mathcal{P}(x) \wedge \mathcal{Q}(\mathcal{N}(x))$ with a verification tool for the decision problem. If the result is $\texttt{UNSAT}$, then the property holds in the whole input space, and the algorithm returns a VR of $0\%$. Otherwise, if the verification tool returns $\texttt{SAT}$, at least one single violation point exists, but we cannot assert how many similar configurations exist in the input space. To this end, the algorithm checks another property equivalent to the original one, thus negating the postcondition (i.e., $\neg \mathcal{Q} = v_1^4 \geq 0$). 
Here, we have two possible outcomes:
\begin{itemize}
    \item $\mathcal{P}(x_0) \wedge \neg\mathcal{Q}(\mathcal{N}(x_0))$ is $\texttt{UNSAT}$ implies that all possible $x_0 = (v_1^1, v_1^2)$ that satisfy $\mathcal{P}$, output a value strictly less than 0, violating the safety property. Hence, the algorithm returns a 100\% of \textit{VR} in the input area represented by $\mathcal{P}$. This situation is depicted with the red circle in Fig. \ref{fig:exactCount}.
    \item $\mathcal{P}(x_0) \wedge \neg\mathcal{Q}(\mathcal{N}(x_0))$ is $\texttt{SAT}$ implies that there is at least one  input configuration ${\bf v} = (v_1^1, v_1^2)$ satisfying $\mathcal{P}$ and such that $\mathcal{N}({\bf v}) \geq 0.$ Therefore, the algorithm cannot assert anything about the violated portion of the area since there are some input points on which $\mathcal{N}$ generates outputs greater or equal to zero and some others that generate a result strictly less than zero. Hence, the procedure splits the input space into two parts, as depicted in Fig.~\ref{fig:exactCount}. 
\end{itemize}

This process is repeated until the algorithm reaches one of the base cases, such as a situation in which either $\mathcal{P}(x) \wedge \mathcal{Q}(\mathcal{N}(x))$ is not satisfiable (i.e., all the current portion of the input space is safe), represented by a green circle in Fig.\ref{fig:exactCount}), or $\mathcal{P}(x) \wedge \neg\mathcal{Q}(\mathcal{N}(x))$ is not satisfiable 
(i.e., the whole current portion of the input space is unsafe), represented by a red circle. Note that the option of obtaining $\texttt{UNSAT}$ on both properties is not possible.

Finally, the algorithm of exact count returns the \textit{VR} as the ratio between the unsafe areas (i.e., the red circles) and the original input area. In the example of Fig.~\ref{fig:exactCount}, assuming a 2-decimal-digit discretization, we obtain a total number of violation points equal to $8951$. Normalizing this value by the initial total number of points ($10201$), and considering the percentage value, we obtain a final \textit{VR} of $87.7 \%$.
Moreover, the \textit{Violation Rate} can also be interpreted as the probability of violating a safety property using a particular DNN $\mathcal{N}$. In more detail, uniformly sampling $10$ random input vectors for our DNN $\mathcal{N}$ in the intervals described by $\mathcal{P}$, $8$ over $10$ with high probability violates the safety property.

\subsection{Hardness of \#DNN-Verification} \label{complexityExactCount}

It is known that finding even just one input configuration (also referred to as violation point) that violates a given safety property is computationally hard since the \textit{DNN-Verification} problem is \textit{NP-Complete} \cite{Reluplex}. Hence, counting and enumerating all the violation points is expected to be an even harder problem. 
\begin{theorem}\label{th:p-complete}
    The \textit{\#DNN-Verification} is \textit{\#P-Complete}.
\end{theorem}
Note that, in general, the counting version of an \textit{NP-complete} problem is not necessarily \#P-complete. We are able to prove Theorem \ref{th:p-complete} by showing that the hardness proof for \textit{DNN-Verification} in \cite{Reluplex} provides a bijection between an assignment that satisfies a \textit{3-CNF} formula and a violation point in given input space (we include the details is in the supplementary material). Hence, \textit{\#DNN-Verification} turns out to be \textit{\#P-Complete} as \textit{\#3-SAT} \cite{valiant1979counting}.

\section{CountingProVe for Approximate Count}
\label{countingProve}

In view of the \#P-completeness of the \#DNN-Verification problem, it is not surprising that the time complexity of the algorithm for the exact count worsens very fast when the size of the instance increases. In fact also moderately large networks are intractable with this approach.   
To overcome these limitations while still providing guarantees on the quality of the results, we propose a randomized-approximation algorithm, \texttt{CountingProVe} (presented in Algorithm~\ref{alg:CountingProVe}). 

To understand the intuition behind our approach, consider Fig. \ref{fig:exactCount}. If we could assume that each split distributed the violation points evenly in the two subinstances it produces, for computing the number of violation points in the whole input space it would be enough to compute the number of violation points in the subspace of a single leaf and multiplying it by $2^s$ (which represents the number of leaves). Since we cannot guarantee a perfectly balanced distribution of the violation points among the leaves,
we propose to use a heuristic strategy to split the input domain, balancing the number of violation points in expectation. This strategy 
allows us to estimate the count and a provable confidence interval on the actual violation rate. In more detail, our algorithm receives as input the tuple for a DNN-Verification problem (i.e., $\mathcal{R}=\langle\mathcal{N}, \mathcal{P}, \mathcal{Q}\rangle$) and to obtain a precise estimate of the \textit{VR}, performs $t$ iterations of the following procedure. 

\begin{algorithm}[t]
\caption{\texttt{CountingProVe}}\label{alg:CountingProVe}
\begin{algorithmic}[1]
\small
\STATE \textbf{Input:} $\mathcal{R} =\langle\mathcal{N},\mathcal{P}, \mathcal{Q}\rangle$, $t$ ($\#$ of repetitions), $m$ ($\#$ of violation points sampled per iteration), $\beta > 0$ (error tolerance factor).
\STATE \textbf{Output: } lower bound of the violation rate.
\vspace{0.2cm}

\FOR{t=1 to $t$}
    \STATE  $VR_t$ $\gets 100\%,$ $s \gets 0$ 
    \WHILE{$Timeout(ExactCount(\mathcal{R}))$} 
        \STATE $S \gets SampleViolationPoints(\mathcal{R}, m)$
        \STATE $median \gets ComputeMedian(S, node_i)$
        \STATE $node_{i,0}, node_{i,1} \gets SplitInterval(node_i, median)$
        \STATE $side \gets$ a random value chosen uniformly from \{0,1\}
        \STATE $\mathcal{P} \gets UpdateP(\mathcal{P}, node_{i,side})$
        \STATE $s \gets s + 1$
    \ENDWHILE
    \STATE $VR_t \gets 2^{s - \beta} \cdot$ $ExactCount(\mathcal{R})\cdot \prod_{i=1}^{s} \alpha_{i}$
    \STATE $VR \gets min(VR_t, VR)$
    
\ENDFOR
\STATE \textbf{return} \textit{VR}
\end{algorithmic}
\end{algorithm}

Initially, we assume that, in the whole input domain, the safety property does not hold, i.e., we have a 100\% of \textit{Violation Rate}. The idea is to repeatedly simplify the input space by performing a sequence of multiple splits, where each $i$-th split implied a reduction of the input space of a factor $\alpha_i$. At the end of $s$ simplifications, the goal is to obtain an exact count of unsafe input configurations that can be used to estimate the VR of the entire input space. 
Specifically, Algorithm\;\ref{alg:CountingProVe} presents the pseudo-code for the heuristic approach. 
Before splitting, the algorithm attempts to use the exact count but stops if not completed within a fixed timeout which we set to just a fraction of a second (line 5). 
Inside the loop, the procedure $SampleViolationPoints(\mathcal{R}, m)$ samples $m$ violation points from the subset of the input space described in $\mathcal{P}$ (using a uniform sampling strategy) and saves them in $S$.

After line 6, the algorithm has an estimation of how many violation points there are in the portion of the input space under consideration\footnote{if the $m$ solutions are not found, the algorithm proceeds by considering only the violation points discovered or splitting at random.}. The idea is to simplify one dimension of the hyperrectangle cyclically while keeping the number of violation points balanced in the two portions of the domain we aim to split (i.e., as close to $\vert S \vert/2$ as possible). Specifically,
$ComputeMedian(S, node_i)$ (line 7) computes the median of the values along the dimension of the node chosen for the split. $SplitInterval(node_i, median)$ splits the node into two parts $node_{i,0}$ and $node_{i,1}$ according to the median computed. For instance, suppose we have an interval of $[0,1]$ and the median value is $0.33$. The two parts obtained by calling $SplitInterval(node_i, median)$ are $[0, 0.33]$ and $(0.33, 1]$. The algorithm proceeds randomly, selecting the side to consider from that moment on and discarding the other part. Hence, it updates the input space intervals to be considered for the safety property $\mathcal{P}$ (lines 9-10). Finally, the variable $s$ that represents the number of splits made during the iteration is incremented.
At the end of the while loop (line 12), the VR is computed by multiplying the result of the exact count by $2^{s - \beta}$, and for $\prod_{i=1}^{s} \alpha_i$. The first term ($2^{s - \beta}$) considers the fact that we balanced the \textit{VR} in our tree, hence selecting one path and multiplying by $2^s$, we get a representative estimate of the whole starting input area. $\beta$ is a tolerance error factor, and finally,  $\prod_{i=1}^{s} \alpha_i$ describes how we simplified the input space during the $s$ splits made (line 13). Finally, the algorithm refines the lower bound (line 14).

The crucial aspect of \texttt{CountingProVe} is that even when the splitting method is arbitrarily poor, the bound returned is provably correct (from a probabilistic point of view), see next section for more details. Moreover, the correctness of our algorithm is independent on the number of samples used to compute the median value. However, using a poor heuristic (i.e., too few samples or a suboptimal splitting technique), the bound's quality can change, as the lower bound may get farther away from the actual \textit{Violation Rate}. We report in the supplementary material an ablation study that focuses on the impact of the heuristic on the results obtained.

\subsection{A Provable Lower Bound}\label{lower}
In this section, we show that the randomized-approximation algorithm \texttt{CountingProVe} returns a correct lower bound for the \textit{Violation Rate} with a probability greater (or equal) than $(1 - 2^{-\beta t})$. In more detail, we demonstrate that the error of our approximation decreases exponentially to zero and that it does not depend on the number of violation points sampled during the simplification process nor on the heuristic for the node splitting.

\begin{theorem}\label{th2}
    Given the tuple $\mathcal{R} =\langle\mathcal{N},\mathcal{P}, \mathcal{Q}\rangle$, the \textit{Violation Rate} returned by the randomized-approximation algorithm \texttt{CountingProVe} is a correct lower bound with a probability $ \geq (1 - 2^{-\beta t})$.
\end{theorem}

\begin{proof}
   Let $VR^* > 0$ be the actual violation rate. Then, \texttt{CountingProVe} returns an incorrect lower bound, if for each iteration, $VR > VR^*$. The key of this proof is to show that for any iteration, $Pr(VR > VR^*) < 2^{-\beta}$. 
    
    Fix an iteration $t$. The input space described by $\mathcal{P}$ and under consideration within the while loop is repeatedly simpliﬁed. Assume we have gone through a sequence of $s$ splits, where, as reported in Sec.\ref{countingProve}, the $i$-th split implied a reduction of the solution space of a factor $\alpha_i$. For a violation point $\sigma,$ we let $Y_\sigma$ be a random variable that is 1 if $\sigma$ is also a violation point in the restriction of the input space that has been (randomly) obtained after $s$ splits (and that we refer to as the solution space of the leaf $\ell$).
    
    Hence, the \textit{VR} obtained using an exact count method for a particular $\ell$ is $VR_{\ell} = \frac{\sum_{\sigma\in\Gamma} Y_\sigma}{A_{\ell}}$, i.e., the ratio between the number of violation points and the number of points in $\ell$ ($A_{\ell}$). Then our algorithm returns as an estimate of the total \textit{VR} computed by:
\begin{equation}
    VR = 2^{s-\beta} \cdot VR_\ell \cdot \prod_{i=1}^s \alpha_i
           = 2^{s-\beta} \cdot \frac{\sum_{\sigma\in\Gamma} Y_\sigma}{A_{\ell}} \cdot \prod_{i=1}^s \alpha_i
\end{equation}
%
%
 %
   %
    Let $\textbf{s}$ denote the sequence of $s$ random splits followed to reach the leaf $\ell$. We are interested in $\mathbb{E}[VR] = \mathbb{E}[\mathbb{E}[VR \vert \textbf{s}]]$. The inner conditional expectation can be written as:
    
    \begin{align}
       \mathbb{E}[VR \;\vert \;\textbf{s}] &= \mathbb{E}\Big[ 2^{s - \beta} \cdot \frac{\sum_{\sigma\in\Gamma} Y_\sigma}{A_{\ell}} \cdot \prod_{i=1}^s \alpha_i \;\big\vert \;\textbf{s} \;\Big]\\
       &= \mathbb{E}\Big[2^{s - \beta} \cdot \frac{\sum_{\sigma\in\Gamma} Y_\sigma}{A_{Tot}} \;\big\vert \;\textbf{s} \;\Big]\\
       &= \frac{2^{s -\beta}}{A_{Tot}} \sum_{\sigma} \mathbb{E}[Y_\sigma = 1 \;\vert \;\textbf{s}\;]\\
       &= \frac{2^{s -\beta}}{A_{Tot}} \sum_{\sigma} 2^{-s}  \\
       &= 2^{-\beta} \frac{\sum_{\sigma\in\Gamma} 1}{A_{Tot}} = 2^{-\beta} VR^*
    \end{align}
        
    where the equality in (2) follows from rewriting \textit{VR} using (1). Equality (3) follows from (2) using the relation $A_{Tot} = \frac{A_\ell}{\prod_{i=1}^s \alpha_i}$. Then we have (4) that follows from (3) by using the linearity of the expectation. (5) follows from (4) since, in each split, we choose the side to take independently and with probability $1/2$. Finally, (6) follows by recalling that $\sum_{\sigma\in\Gamma} 1$ is the total number of violations in the whole input space, hence $VR^* = \frac{\sum_{\sigma\in\Gamma} 1}{A_{Tot}}$.
    Therefore, we have
    
        \begin{equation*}
            \mathbb{E}[VR] = \mathbb{E}[\mathbb{E}[VR \vert \textbf{s}]] = \mathbb{E}[2^{-\beta} VR^*] = 2^{-\beta} VR^*.
        \end{equation*}
    Finally, by using Markov's inequality, we obtain that
    \begin{equation*}
        Pr(VR > VR^*) < \frac{\mathbb{E}[VR]}{VR^*} =  \frac{ 2^{-\beta} VR^*}{VR^*} = 2^{-\beta}.
    \end{equation*}
    Repeating the process $t$ times, we have a probability of overestimation equal to $2^{-\beta t}$. This proves that the \textit{Violation Rate} returned by \texttt{CountingProVe} is correct with a probability $ \geq (1 - 2^{-\beta t})$.
\end{proof}

\subsection{A Provable Confidence Interval of the VR}
This section shows how a provable confidence interval can be defined for the \textit{VR} using \texttt{CountingProVe}. In particular, recalling the definition of \textit{Violation Rate} (Def. \ref{violation}), it is possible to define a complementary metric, the \textit{safe rate (SR)}, counting all the input configurations that do not violate the safety property (i.e., provably safe). In particular, we define:
\begin{definition}(\textit{Safe Rate (SR)})
     Given an instance of the \textit{DNN-Verification}  problem $\mathcal{R}=\langle\mathcal{N}, \mathcal{P}, \mathcal{Q}\rangle$
     we define the \textit{Safe Rate} as the ratio between the number of safe points and the total number of points satisfying $\mathcal{P}.$ Formally:
    \[SR = \frac{\vert\{ x\;|\;\mathcal{P}(x) \} \setminus \Gamma(\mathcal{R}) \vert}{\vert\{ x\;|\;\mathcal{P}(x) \}\vert} =
    \frac{\vert\Gamma(\mathcal{N}, \mathcal{P}, \neg\mathcal{Q}) \vert}{\vert\{ x\;|\;\mathcal{P}(x) \}\vert}
    \]  
\end{definition}
\noindent where the numerator indicates the sum of non-violation points in the input space. 
From the second expression, it is easy to see that \texttt{CountingProVe} can be used to compute a lower bound of the \textit{SR}, by running
Alg.~\ref{alg:CountingProVe}
on the instance $(\mathcal{N}, \mathcal{P}, 
\neg \mathcal{Q}).$ 

\begin{theorem}\label{th3}
    Given a Deep Neural Network $\mathcal{N}$ and a safety property $\langle\mathcal{P}, \mathcal{Q}\rangle$, complementing the lower bound of the \textit{Safe Rate} obtained using \texttt{CountingProVe}, which is correct with a probability $ \geq (1 - 2^{-\beta t})$, we obtain an upper bound for the \textit{Violation Rate} with the same probability.
\end{theorem}

\begin{proof}
    Suppose we compute the \textit{Safe Rate} with \texttt{CountingProVe}. From Theorem \ref{th2}, we know that the probability of overestimating the real \textit{SR} tends exponentially to zero as the iterations $t$ grow. Hence, $Pr(SR > SR^*) < 2^{-\beta t}$.
    We now consider the \textit{Violation Rate} as the complementary metric of the \textit{Safe Rate}, and we write $VR = 1 - SR$ (as the $SR$ is a value in the interval $[0, 1]$). We want to show that the probability that the \textit{VR}, computed as $VR = 1 - SR$, underestimates the real \textit{Violation Rate} ($VR^*$) is the same as theorem \ref{th2}.

    Suppose that at the end of $t$ iterations, we have a $VR < VR^*$. This would imply by definition that $1-SR < 1 - SR^*$, i.e., that $SR > SR^*$. However, from Theorem \ref{th2}, we know that the probability that \texttt{CountingProVe} returns an incorrect lower bound at each iteration is $2^{-\beta t}$. Hence, we obtained that the \textit{VR} computed as $VR = 1 - SR$ is a correct upper bound with the probability (1 - $2^{-\beta t}$) as desired.
\end{proof}

These results allow us to obtain a confidence interval for the \textit{Violation Rate}\footnote{notice that the same formulation holds for the complementary metrics (i.e., safe rate).}, namely, from Theorems \ref{th2} and \ref{th3} we have:

\newtheorem{lemma}[theorem]{Lemma}
\begin{lemma}
    \texttt{CountingProVe} can compute both a correct lower and upper bound, i.e., a correct confidence interval for the \textit{VR}, with a probability $\geq (1-2^{-\beta t})$.
\end{lemma}


\subsection{CountingProVe is Polynomial} 
\label{polynomiality}
To conclude the analysis of our randomized-approximation algorithm, we discuss some additional requirements to guarantee that our approach runs in polynomial time. Let $N$ denote the size of the instance. It is easy to see that by choosing both the timeout applied to each run of the exact count algorithm used in line 5 and the number $m$ of points sampled in line 6 to be polynomially bounded in $N$ then each iteration of the for loop is also polynomial in $N.$ 
Moreover, if each split of the input area encoded by $\mathcal{P}$ performed 
in lines 7-10 is guaranteed to reduce the size of the instance by at least some 
fraction $\gamma \in (0,1),$ then 
after $s = \Theta(\log N)$ splits the instance in the leaf $\ell$ has size $O(\frac{N}{2^s}) = O(1)$. Hence, we can exploit an exponential time formal verifier to solve the instance in the leaf $\ell$, and the total time will be polynomial in $N$.

\begin{table*}

\tiny
\resizebox{\textwidth}{!}{
    \begin{tabular}{|cl||cc||ccc|}
    \hline
    
    \multicolumn{2}{|c||}{\textbf{Instance}} &
      \multicolumn{2}{c||}{\textbf{Exact Count}} &
      \multicolumn{3}{c|}{\textbf{$\texttt{CountingProVe}$} (confidence $\geq 99\%$)} \\
    \multicolumn{1}{|l}{} &
       &
      \multicolumn{2}{c||}{$\texttt{BaB}$} &
      \multicolumn{1}{l}{} &
      \multicolumn{1}{l}{} &
      \multicolumn{1}{l|}{} \\
    \multicolumn{1}{|l}{}      &       & Violation Rate & \multicolumn{1}{c||}{Time} & Interval confidence VR & Size   & Time    \\[0.2ex] \hline 
    \multicolumn{2}{|l||}{\vspace{-2.5mm}} & & \multicolumn{1}{c||}{} & & & \\
    \multicolumn{2}{|c||}{Model\_2\_20} & 20.78\%  & \multicolumn{1}{c||}{234 min} & {[}19.7\%, 22.6\%{]}   & 2.9\%  & 42 min  \\ [0.5ex]
    \multicolumn{2}{|c||}{Model\_2\_56} & 55.22\%  & \multicolumn{1}{c||}{196 min} & {[}54.13\%, 57.5\%{]}   & 3.37\%  & 34 min  \\[0.5ex]
    \multicolumn{2}{|c||}{Model\_2\_68} & 68.05\%  & \multicolumn{1}{c||}{210 min} & {[}66.21\%, 69.1\%{]}   & 2.89\%  & 45 min  \\[0.5ex] 
    \multicolumn{2}{|c||}{Model\_5\_09} & --  & \multicolumn{1}{c||}{24 hrs} & {[}8.42\%, 13.2\%{]}   & 4.78\% & 122 min \\[0.5ex] 
    \multicolumn{2}{|c||}{Model\_5\_50} & --  & \multicolumn{1}{c||}{24 hrs} & {[}48.59\%, 52.22\%{]}   & 3.63\%  & 124 min \\[0.5ex]
    \multicolumn{2}{|c||}{Model\_5\_95} & --  & \multicolumn{1}{c||}{24 hrs} & {[}91.73\%, 96.23\%{]}   & 4.49\%  & 121 min \\[0.5ex] 
    \multicolumn{2}{|c||}{Model\_10\_76} & --  & \multicolumn{1}{c||}{24 hrs} & {[}74.25\%, 77.23\%{]}   & 3.98\%  & 300 min \\[0.5ex] \Xhline{2.5\arrayrulewidth}
    \multicolumn{2}{|l||}{\vspace{-2.5mm}} & & \multicolumn{1}{c||}{} & & &  \\
    \multicolumn{2}{|l||}{$\phi_2$ ACAS Xu\_2.1}  & -- & \multicolumn{1}{c||}{24 hrs} & {[}0.45\%, 5.01\%{]}   & 4.56\%  & 246 min \\[0.5ex] 
    \multicolumn{2}{|l||}{$\phi_2$ ACAS Xu\_2.3}  & -- & \multicolumn{1}{c||}{24 hrs} & {[}1.23\%, 4.21\%{]}   & 2.98\%  & 241 min \\[0.5ex] 
    \multicolumn{2}{|l||}{$\phi_2$ ACAS Xu\_2.4}  & -- & \multicolumn{1}{c||}{24 hrs} & {[}0.74\%, 3.43\%{]}   & 2.68\%  & 243 min \\[0.5ex] 
    \multicolumn{2}{|l||}{$\phi_2$ ACAS Xu\_2.5}  & -- & \multicolumn{1}{c||}{24 hrs} & {[}1.67\%, 4.10\%{]}   & 2.42\%  & 240 min \\[0.5ex] 
    \multicolumn{2}{|l||}{$\phi_2$ ACAS Xu\_2.7}  & --  & \multicolumn{1}{c||}{24 hrs} & {[}2.35\%, 5.22\%{]}   & 2.87\%  & 240 min \\\hline
    \end{tabular}
}
\caption{Comparison of $\texttt{CountingProVe}$ and exact counter on different benchmark setups. The first block shows the results on our benchmark properties, where every instance is in the form Model\_{$\rho$\_$\psi$}, where $\rho$ is the size of the input space and $\psi$ is the id of the specific DNN. The last block reports the results on the Acas Xu $\phi_2$ benchmark. Full results on $\phi_2$ and another property in the Appendix\ref{fullACASExp}.}\label{tab:results}
\end{table*}
\section{Experimental Results}
\label{results}

In this section, we guide the reader to understand the importance and impact of this novel encoding for the verification problem of deep neural networks. In particular, in the first part of our experiments, we show how the problem's computational complexity impacts the run time of exact solvers, motivating the use of an approximation method to solve the problem efficiently. In the second part, we analyze a concrete case study, ACAS Xu \cite{Reluplex}, to explain why finding all possible unsafe configurations is crucial in a realistic safety-critical domain.
All the data are collected on a commercial PC running Ubuntu 22.04 LTS equipped with Nvidia RTX 2070 Super and an Intel i7-9700k. In particular, for the exact counter, we rely as backend on the formal verification tool \textit{BaB} \cite{BaB} available on \href{https://github.com/sisl/NeuralVerification.jl}{``NeuralVerification.jl"} and developed as part of the work of \cite{Liu}. While, as exact count for $\texttt{CountingProVe}$, we rely on \textit{ProVe} \cite{ProVe} given its key feature of exploiting parallel computation on GPU.
In our experiments with \texttt{CountingProVe}, we set $\beta=0.02$ and $t=350$ in order to obtain a correctness confidence level greater or equal to $99\%$ (refer to Theorem~\ref{th2}). 

Table~\ref{tab:results} summarizes the results of our experiments, evaluating the proposed algorithms (i.e., the exact counter and \texttt{CountingProVe}) on different benchmarks. Our results show the advantage of using an approximation algorithm to obtain a provable estimate of the portion of the input space that violates a safety property. We discuss the results in detail below.  The code used to collect the results and several additional experiments and discussions on the impact of different hyperparameters and backends for our approximation are available in the supplementary material (available \href{https://shorturl.at/cnozV}{here}).

\paragraph{Scalability Experiments}\label{exp_scalability}
In the first two blocks of Tab.\;\ref{tab:results}, we report the experiments related to the scalability of the exact counters against our approximation method $\texttt{CountingProVe}$, showing how the \#DNN-Verification problem becomes immediately infeasible, even for small DNNs.
In more detail, we collect seven different random models (i.e., using random seeds) with different levels of violation rates for the same safety property, which consists of all the intervals of $\mathcal{P}$ in the range $[0, 1]$, and a postcondition $\mathcal{Q}$ that encodes a strictly positive output. In the first block, all the models have two hidden layers of 32 nodes activated with $ReLU$ and two, five, and ten-dimensional input space, respectively. Our results show that for the models with two input neurons, the exact counter returns the violation rate in about $3.3$ hours, while our approximation in less than an hour returns a provable tight ($\sim 3\%$) confidence interval of the input area that presents violations. Crucially, as the input space grows, the exact counters reach the timeout (fixed after 24 hours), failing to return an exact answer. $\texttt{CountingProVe}$, on the other hand, in about two hours, returns an accurate confidence interval for the violation rate, highlighting the substantial improvement in the scalability of this approach. In the supplementary material, we report additional experiments and discussions on the impact of different hyperparameters for the estimate of $\texttt{CountingProVe}$.

\paragraph{ACAS Xu Experiments}\label{resultsACAS}
The ACAS Xu system is an airborne collision avoidance system for aircraft considered a well-known benchmark and a standard for formal verification of DNNs \cite{Liu}\cite{Reluplex}\cite{Neurify}. It consists of 45 different models composed of an input layer taking five inputs, an output layer generating five outputs, and six hidden layers, each containing $50$ neurons activated with $ReLU$. To show that the count of all the violation points can be extremely relevant in a safety-critical context, we focused on the property $\phi_2$, on which, for 34 over 45 models, the property does not hold. In detail, the property $\phi_2$ describes the scenario where if the intruder is distant and is significantly slower than the ownship, the score of a Clear of Conflict (COC) can never be maximal (more detail are reported here \cite{Reluplex}).
We report in the last block of Tab.\;\ref{tab:results} the results of this evaluation. To understand the impact of the \#DNN-Verification problem, let us focus, for example, on the \textit{ACAS Xu\_2.7} model. As reported in Tab.\;\ref{tab:results}, \texttt{CountingProVe} returns a provable lower bound for the violation rate of at least a $2.35\%$. This means that assuming a 3-decimal-digit discretization, our approximation counts at least $23547$ violation points compared to a formal verifier that returns a single counterexample (i.e., a single violation point). Note that a state-of-the-art verifier that returns only \texttt{SAT} or \texttt{UNSAT} does not provide any information about the amount of  possible unsafe configurations considered by the property.

\section{Discussion}

In this paper, we first present the \textit{\#DNN-Verification}, the problem of counting the number of input configurations that generates a violation of a given safety property. We analyze the complexity of this problem, proving the \#P-completness and highlighting why it is relevant for the community. Furthermore, we propose an exact count approach that, however, inherits the limitations of the formal verification tool exploited as a backend and struggles to scale on real-world problems. Crucially, we present an alternative approach, \texttt{CountingProVe}, which provides an approximated solution with formal guarantees on the confidence interval. Finally, we empirically analyze our algorithms on a set of benchmarks, including a real-world problem, ACAS Xu, considered a standard in the formal verification community.

Moving forward, we plan to investigate possible optimizations in order to improve the performance of \texttt{CountingProVe}, for example, by improving the node selection and the bisection strategies for the interval; or by exploiting the result of \#DNN-Verification to optimize the system's safety during the training loop. 
\bibliographystyle{named}
\bibliography{ijcai23}

\clearpage
\appendix
\section*{Appendix: Supplementary Material}
\section{\#DNN-Verification is \#P-Complete}\label{DNNPC}

\setcounter{theorem}{0}
\begin{theorem}
    The \textit{\#DNN-Verification} problem is \textit{\#P-Complete}.
\end{theorem}

\begin{proof}
The proof of \#P-completeness is similar and follows the one of NP-Completeness; the main difference is in the concept of polynomial time \textit{counting reduction}. As stated in \cite{valiant1979counting} and \cite{gomes2021model}, many NP-complete problems are \textit{parsimonious}, meaning that for almost all pairs of NP-complete problems, there exist polynomial transformations between them that preserve the number of solutions. Hence, the reduction between two NP-Complete problems can be directly taken as part of a counting reduction, thus providing an easy path to proving \#P-completeness.

For our purpose, we follow the reduction between \textit{3-SAT} and \textit{DNN-Verification} provided in the work of \cite{Reluplex}. In more detail, we assume that the input nodes take the discrete values in $\{0,1\}$ for simplicity. Note that this limitation can be relaxed using an $\varepsilon$ discretization to consider a range between $[a, b]$ for the input space. 

Recalling the hardness proof, we know that any 3-SAT formula $\Phi$ can be transformed into a DNN $\mathcal{N}$ (with ReLUs activation functions) and a property $\phi$, such that $\phi$ is satisfiable on $\mathcal{N}$ if and only if $\Phi$ is satisfiable. Specifically, \cite{Reluplex} provided three useful gadgets to perform the reduction: 
\begin{enumerate}
    \item \textbf{disjunction gadget} that maps a disjunction of three literals in a 3-CNF formula to the same result for a group of three nodes in a DNN. Formally this gadget performs the following transformation: $y_i = 1 - \max(0, 1 - \sum_{j=1}^3 q_i^j)$. Where $y_i$ is the node that collects the result of the linear combination and subsequent $ReLU$ activation of up to 3 nodes ($q_i^j$) from the previous layer.  Hence, $y_i$ will be $1$ if  at least one input variable is set to 1 (or true), and $y_i$ will be $0$ if all input variables are set to $0.$  In words, this gadget maps a disjunction of literals in a 3-CNF to a combination of nodes in a DNN, such that there is a one-one correspondence between the output of the nodes on a 0-1 input and the truth value computed by the disjunction over the equivalent truth values. 
    \item \textbf{negation gadget} that on input $x_i \in \{0,1\}$ produces the output value 
    $y_i = 1 - x_j,$ hence modelling the exact behaviour of a logical negation.
    \item \textbf{conjuction gadget} which maps the satisfiability of a 3-CNF $\Phi$ into a $\phi$ satisfiability for a DNN. In particular, $\Phi$ is satisfied only if all clauses $C_1, \dots, C_n$ are simultaneously satisfied. Hence, if all the nodes are in the domain $\{0,1\}$, for satisfiability, we want the resulting output of a forward propagation equal to $n$, i.e., the number of clauses. This gadget maps the conjunction of $n$ clauses in a 3-CNF, i.e., the satisfiability, into the linear combination of $n$ nodes to produce an output value.
    Therefore, $C_1 \wedge C_2 \wedge 
    \cdots \wedge C_n = true$ if and only if the output of the gadget is $n.$
\end{enumerate}

From the combination of these three gadgets, we obtain a reduction transforming a 3-CNF formula $\phi$ into a DNN $\cal N$. Let us consider the instance to the DNN-Verification problem asking to check whether there
exists an input configuration on which $\cal N$ outputs a value different from $n$. Then, we have that the formula $\phi$ is satisfiable, i.e., there is a truth assignment to the input variables if and only if for the DNN $\cal N$ there exists an input configuration (in fact necessarily only using values in $\{0,1\}$) that induces the output $y = n$, i.e., if and only if, there exists a violation.


As observed, this reduction also shows that each distinct satisfying assignment for $\phi$ is mapped to a distinct input configuration producing output $n$ and 
vice versa, each input configuration
on which $\cal N$ outputs $n$ must be 
$\{0,1\}$-valued and corresponds to a truth assignment that satisfies $\phi.$

Therefore counting the number of satisfying assignments for $\phi$ is equivalent to counting the number of violations for $\cal N.$ Hence, from the \#P-Completness of \#3SAT \cite{valiant1979counting} it follows (via the above  reduction) that also \textit{\#DNN-Verification} is \#P-Complete. 


\end{proof}

\begin{table*}

\tiny
\resizebox{\textwidth}{!}{
    \begin{tabular}{|cl||cc||ccc|}
    \hline
    
    \multicolumn{2}{|c||}{\textbf{Confidence}} &
      \multicolumn{2}{c||}{\textbf{Hyperparameters}} &
      \multicolumn{3}{c|}{\textbf{\texttt{CountingProVe}}} \\
    \multicolumn{1}{|l}{}      &       & $\beta$ & \multicolumn{1}{c||}{$t$} & Interval confidence VR & Size   & Time    \\[0.2ex] \hline 
    \multicolumn{2}{|l||}{\vspace{-2.5mm}} & & \multicolumn{1}{c||}{} & & & \\
    \multicolumn{2}{|c||}{} & 0.02  & \multicolumn{1}{c||}{350} & {[}54.13\%, 57.5\%{]}   & 3.37\%  & 34 min  \\[0.5ex]
    \multicolumn{2}{|c||}{99\%} & 0.1  & \multicolumn{1}{c||}{70} & {[}51.36\%, 59.17\%{]}   & 7.81\%  & 8 min  \\[0.5ex]
    \multicolumn{2}{|c||}{} & 1.5  & \multicolumn{1}{c||}{5} & {[}19.39\%, 84.35\%{]}   & 64.96\%  & 32 sec  \\[0.5ex] \Xhline{2.5\arrayrulewidth} 
    \multicolumn{2}{|l||}{\vspace{-2.5mm}} & & \multicolumn{1}{c||}{} & & & \\
    \multicolumn{2}{|c||}{} & 0.02  & \multicolumn{1}{c||}{170} & {[}52.99\%, 56.68\%{]}   & 3.69\%  & 17 min \\[0.5ex]
    \multicolumn{2}{|c||}{90\%} & 0.1  & \multicolumn{1}{c||}{34} & {[}51.37\%, 58.93\%{]}   & 7.55\% & 4 min \\[0.5ex] 
    
    \multicolumn{2}{|c||}{} & 1.5  & \multicolumn{1}{c||}{3} & {[}19.53\%, 84.32\%{]}   & 64.79\%  & 18 sec \\[0.5ex] \Xhline{2.5\arrayrulewidth}
    \multicolumn{2}{|l||}{\vspace{-2.5mm}} & & \multicolumn{1}{c||}{} & & &  \\
    \multicolumn{2}{|l||}{}  & 0.02 & \multicolumn{1}{c||}{137} & {[}53.52\%, 57.01\%{]}   & 3.48\%  & 14 min \\[0.5ex] 
    \multicolumn{2}{|c||}{85\%}  & 0.1 & \multicolumn{1}{c||}{27} & {[}51.47\%, 58.67\%{]}   & 7.2\%  & 3 min \\[0.5ex] 
    \multicolumn{2}{|l||}{}  & 1.5 & \multicolumn{1}{c||}{2} & {[}19.56\%, 84.24\%{]}   & 64.67\%  & 12 sec \\ \hline
    \end{tabular}
}
\caption{Comparison of different hyperparameters for $\texttt{CountingProVe}$ on \textit{Model\_2\_56}. The true VR is equal to $55.22\%$.}\label{tab:hyperparam}
\end{table*}

\section{Hyperparameters and Ablation Study}\label{ablation}

We report in this section the hyperparameters used to collect the results shown in the main paper. Regarding the heuristic presented in the Alg.\;\ref{alg:CountingProVe}, we want to point out to the reader that a possible optimization is to perform a fixed number of $s$ simplifications before calling the exact count method. In fact, as shown in the main paper, given the complexity of the problem, calling the exact count too frequently when input space is still considerably large typically results in a timeout, thus causing a waste of computation and time. For this reason, we decided to perform a fixed number of preliminary simplifications before calling the exact count.
In more detail, we set $s=17$ for the first row of Tab.\;\ref{tab:results}, and $s=45$ for the remaining part. The value $s$ for the preliminary simplifications can be obtained assuming any discretization of the initial input space $N$, described by $\mathcal{P}$. Hence, relying on the considerations discussed in the Sec. \ref{polynomiality} for the polynomiality of the approximation, we set $s = \lfloor \log N \rfloor - 1$ to ensure the termination of an exponential time exact counter. Moreover, to collect the data of Tab.\;\ref{tab:results} and \ref{tab:fullResultsACAS}, as stated in the main paper, we set $\beta=0.02, t=350$ obtaining a confidence level of $99\%$ (see theorem \ref{th2}). Finally, regarding the number of samples to compute the median, we set $m=1.5M$ for the scalability experiment of Sec.~\ref{exp_scalability} and $m=3M$ for the ACAS Xu experiments.

We performed additional experiments to highlight the impact of different hyperparameters on the quality of the estimate returned by \texttt{CountingProVe}. Crucially, as specified in section \ref{countingProve} in the main paper, the correctness of the algorithm is independent of the heuristic used by the algorithm. We now analyze the impact on the estimate of different parameters such as $\beta, t$, and $m$.

\subsection*{Experiments on Different $\beta$ and $t$}
Tab.\;\ref{tab:hyperparam} shows the comparison results between different hyperparameters for \texttt{CountingProVe}. In detail, all the experiments are performed on the same model ``\textit{Model\_2\_56}", using $s=17$ preliminary simplifications before calling the exact count. Moreover, we use the same number of $m = 1.5M$ samples to compute the median value in the heuristic. We test three confidence levels at 85\%, 90\%, and 99\%, respectively, setting three possible value pairs for $\beta$ and $t$. 

Regarding the impact of the error tolerance factor $\beta$, as expected, as this value increases, the confidence interval deteriorates.
We justify this as the tolerance factor appears in the formula for calculating the violation at the end of each while loop ($2^{s-\beta}\cdot ExactCount$). Hence, a larger $\beta$ strongly impacts the value of the estimate, thus also potentially deteriorating the lower and upper bounds. 
However, we want to emphasize once again as for any value tested at a high confidence level, the estimate returned by \texttt{CountingProVe} is correct, i.e., the lower bound does not overestimate, and the upper bound never underestimates the value returned by the true count (equals to $55.22\%$).

Interestingly, we note that by setting the same value for the error tolerance factor ($\beta=0.02$), the estimate for the three confidence levels is quite similar. Our approximation thus allows choosing the desired confidence level while obtaining a good estimate and potentially saving time. In fact, by choosing a confidence of 85\%, we obtain an estimate very close to the best estimate obtained with 99\% confidence, halving the computation time.

\subsection*{Impact of the $m$ Samples in CountingProVe}
Although the correctness of the approximation does not depend on the number of violation samples to compute the median value (as shown in Theorem \ref{th2}), we performed an additional experiment to understand its impact on the quality of the estimation. The experiment was performed on the ``\textit{Model\_2\_56}" model with parameters $\beta=0.02, t=350, s=17$. We report in Tab\;\ref{tab:m} the comparison of four different sample values, $500k, 1M, 1.5M, 3M$, and finally $5M$. As we can notice, increasing the sampling size $m$ to find the violation points to compute the median leads to a more accurate estimate of the true violation rate. Intuitively, we obtain a (theoretically) higher probability of finding violation points in the input space by increasing the number of samples. Hence, the more violation points we randomly sample, the more information we obtain to compute an accurate median. However, using more samples results in more time to compute the median and, consequently, the confidence interval of the violation rate. 

\begin{table}[ht!]

\resizebox{\linewidth}{!}{
    \begin{tabular}{|c||ccc|}
    \hline
    
    \multicolumn{1}{|c||}{\textbf{$m$ samples}} &
      \multicolumn{3}{c|}{\textbf{\texttt{CountingProVe}} (confidence $\geq$ 99\%)} \\
    \multicolumn{1}{|c||}{}  & Interval confidence VR & Size   & Time    \\[0.2ex] \hline 
    \multicolumn{1}{|l||}{\vspace{-2.5mm}} & & & \\
    \multicolumn{1}{|c||}{$500k$} & {[}52.59\%, 66.36\%{]}   & 13.8\%  & 13 min  \\[0.5ex]
    \multicolumn{1}{|c||}{$1M$} & {[}51.45\%, 57.2\%{]}   & 5.74\%  & 24 min  \\[0.5ex]
    \multicolumn{1}{|c||}{$1.5M$} & {[}54.13\%, 57.5\%{]}   & 3.37\%  & 34 min  \\[0.5ex]
    \multicolumn{1}{|c||}{$3M$} & {[}53.41\%, 56.63\%{]}   & 3.21\%  & 40 min \\[0.5ex]
    \multicolumn{1}{|c||}{$5M$} & {[}54.17\%, 56.42\%{]}   & 2.24\%  & 60 min \\\hline
    \end{tabular}
}
\caption{Comparison of different $m$ for $\texttt{CountingProVe}$}
\label{tab:m}
\end{table}

\begin{table*}

\tiny
\resizebox{\textwidth}{!}{
    \begin{tabular}{|cl||cccc||cccc|}
    \hline
    
    \multicolumn{2}{|c||}{\textbf{Instance}} &
      \multicolumn{4}{c||}{\textbf{Hyperparameters}} & \multicolumn{4}{c|}{\textbf{\texttt{CountingProVe}}}\\
    \multicolumn{1}{|l}{}  & & $\beta$ & $t$ & $s$ & $m$ & Backend & Interval VR & Size   & Time    \\[0.2ex] \hline 
    \multicolumn{2}{|l||}{\vspace{-2mm}} & & \multicolumn{1}{c}{} & & & & & & \\

    \multicolumn{2}{|c||}{\textit{Model\_2\_56}} & 0.1  & 70 & 15 &  1.5M & \texttt{BaB}& {[}51.36\%, 59.17\%{]}   & 7.81\%  & 10 min  \\[0.5ex] 
    \multicolumn{2}{|c||}{\textit{Model\_2\_56}} & 0.1  & 70 & 15 &  1.5M & \texttt{ProVe}& {[}51.36\%, 59.17\%{]}   & 7.81\%  & 8 min \\\hline
        \multicolumn{2}{|l||}{\vspace{-2mm}} & & \multicolumn{1}{c}{} & & & & & & \\
    \multicolumn{2}{|c||}{\textit{Model\_2\_56}} & 0.02  & 350 & 22 &  1.5M & \texttt{BaB}& {[}53.77\%, 57.03\%{]}   & 3.26\%  & 60 min  \\[0.5ex] 
    \multicolumn{2}{|c||}{\textit{Model\_2\_56}} & 0.02  & 350 & 22 &  1.5M & \texttt{ProVe}& {[}53.77\%, 57.03\%{]}   & 3.26\%  & 50 min \\\hline
     \multicolumn{2}{|l||}{\vspace{-2mm}} & & \multicolumn{1}{c}{} & & & & & & \\
    \multicolumn{2}{|c||}{\textit{Model\_5\_95}} & 0.02  & 350 & 79 &  1.5M & \texttt{BaB}& {[}92.42\%, 96.22\%{]}   & 3.8\%  & 185 min  \\[0.5ex] 
    \multicolumn{2}{|c||}{\textit{Model\_5\_95}} & 0.02  & 350 & 79 &  1.5M & \texttt{NSVerify}& {[}92.42\%, 96.22\%{]}   & 3.8\%  & 180 min   \\[0.5ex] 
    \multicolumn{2}{|c||}{\textit{Model\_5\_95}} & 0.02  & 350 & 79 &  1.5M & \texttt{ProVe}& {[}92.42\%, 96.22\%{]}   & 3.8\%  & 150 min  \\\hline
    \end{tabular}
}
\caption{Comparison of different backends for $\texttt{CountingProVe}$}\label{tab:backends}
\end{table*}

\section{DNN-Verification and Tools}\label{related}
Due to the increasing adoption of DNN systems in safety-critical tasks, the formal method community has developed many verification methods and tools. In literature, these approaches are commonly subdivided into two categories: (i) search-based and (ii) SMT-based methods \cite{Liu}. The algorithm from the first class typically relies on the interval analysis \cite{Moore} to propagate the input bound through the network and perform a reachability analysis in the output layer \cite{BetaCrown}\cite{ProVe}\cite{Neurify}. The second class, in contrast, tries to encode the linear combinations and the non-linear activation functions of a DNN as constrained for an optimization problem \cite{Reluplex}\cite{MILP}\cite{Marabou}.
Crucially, our work is built upon the promising results and the constant improvement in the scalability of these methodologies. In particular, our exact count algorithm is agnostic to the verification tool exploited as the backend and can thus take advantage of any improvement in the field.

In recent years, some effort has also been made to exploit the results of the formal verification analysis in practical application. The work of \cite{TACAS}, for example, proposes a methodology to provide guarantees about the behavior of robotic systems controlled via DNNs; here, the authors exploited a formal verification pipeline to filter the models that respect some hard constraints. Other approaches attempt to improve adherence to some properties as part of the training process, exploiting the results of the formal analysis as a signal to optimize \cite{CROP}\cite{Curriculum}\cite{SOS}\cite{SafeAquatic}. We believe our work can be used to drastically reduce computational time and provide more informative results, encouraging the development of similar approaches to improve the safety of DNN-based systems.

\subsection*{Alternative Backends for CountingProVe}
To show that the correctness of our approximation is independent of the backend chosen, we conducted additional experiments using \textit{BaB} \cite{BaB} and \textit{NSVerify} \cite{NSVerify} as the exact counters instead of \textit{ProVe} \cite{ProVe} for the final count on the leaf in \texttt{CountingProVe}. 
To perform a fair analysis, given the stochastic nature of our approximation, we set the same seed for all methods tested, only changing some hyperparameters and network sizes. We report in Tab.\;\ref{tab:backends} the results of our experiments. 
As expected, the resulting interval of confidence for the VR is the equivalent using any exact counters or hyperparameters in all the tests performed. In more detail, in the first two rows of Tab.\;\ref{tab:backends}, we use the same model (\textit{Model\_2\_56}), only varying the hyperparameters for the confidence (i.e., $\beta$ and $t$), and the number of preliminary splits $s$. We can notice as long as the network size is still small, the use of the GPU (used in \textit{ProVe}) does not bring much benefit. In fact, there is a slight difference in the time to compute the interval of confidence of the VR in both tests with the model with only two input nodes. Moreover, while performing multiple preliminary splits ($s$) can take more time, it also slightly improves the confidence interval. Crucially, notice that in the last two rows of Tab.\;\ref{tab:backends} a little improvement of the interval confidence of VR w.r.t the results presented in Tab.\;\ref{tab:results} and \ref{tab:hyperparam}.

Regarding the last row of Tab.\;\ref{tab:backends}, we used a different model (\textit{Model\_5\_95}) to test the scalability of other exact counters in combination with \texttt{CountingProVe}. The interesting thing to point out in this experiment is that \textit{BaB} (or any different DNN-verification tool) used as a backend for the exact count on the same model results in timeout  (i.e., after 24 hours, it does not return a result) as reported in Tab\;\ref{tab:results}. However, using it as a backend in our approximation, in about 3 hours, can return a very tight confidence interval of the amount of the input space that presents violations. This shows that the intuition behind our approximation brings significant scalability improvements. 

Finally, in this last experiment, we confirm what we mentioned above. As the network grows, having GPU support brings significant improvements in timing, as \textit{ProVe}, in this experiment, saves 30 minutes of computation. Hence, this result motivates us to use it as the ``default" backend for our approximation. Moreover, \textit{ProVe} can verify any DNNs, i.e., with any activation function, which is not typically possible with any state-of-the-art DNN-verification tool. However, this experiment clearly shows that any verifier (perhaps that exploits GPUs) can be employed in \texttt{CountingProVe}, so potentially future improvements or new methods can be easily integrated into our approximation.

\subsection*{Discretization} It is crucial to point out that discretization is not a requirement of our counting approach. In fact, supposing to have a backend that can deal with continuous values, \texttt{CountingProve} would not require such discretization. Nevertheless, the discretization factor might be a parameter of the algorithm. To this end, we performed an analysis of the impact of this parameter, reporting the results in Tab.~\ref{tab:discretization}.
Our experiments demonstrate that using a less fine-grained discretization produces less accurate outcomes, but it enhances the efficiency of the process in terms of time. In the main paper, we opted for a discretization value of 3 (i.e., $0.001$) that provides a good balance between time and accuracy.

\begin{table}[ht!]

\vspace{-2.5mm}
\resizebox{\linewidth}{!}{
    \begin{tabular}{|c||ccc|}
    \hline
    
    \multicolumn{1}{|c||}{\textbf{Rounding}} &
      \multicolumn{3}{c|}{\textbf{\texttt{CountingProVe}} (confidence $\geq$ 99\%)} \\ 
    \multicolumn{1}{|c||}{(decimal digit)} & & Interval Size   & Time    \\[0.2ex] \hline 
    \multicolumn{1}{|l||}{\vspace{-2.5mm}} & & & \\
    \multicolumn{1}{|c||}{$1$} &    & 64.8 $\pm 2.2$\%  & $\sim$213 min  \\[0.5ex]
    \multicolumn{1}{|c||}{$3$} &    & 2.83 $\pm 0.19$\%  & $\sim$242 min  \\[0.5ex]
    \multicolumn{1}{|c||}{$5$} &    &  2.2$\pm 0.38$\%  & $\sim$315 min  \\\hline
    \end{tabular}
}
%
\caption{Different discretization test on property $\phi_2$ of ACAS Xu.}
\label{tab:discretization}
\end{table}

\subsection*{Single Check Verification}
In Tab.~\ref{tab:single_check}, we provide the results of our additional experiments on the single check verification using $\alpha$-$\beta$-$CROWN$\cite{BetaCrown}. We point out that this does not provide the same information as our proposed approach. Specifically, running a decision verifier multiple times does not provide information about the actual number of violations.
Nevertheless, as reported in the main paper, the \texttt{UNSAT} case can be interpreted as a counting result, where the answer is \textit{zero} violations.

\begin{table}[t!]
\small
\centering

\vspace{-3mm}
\begin{tabular}{cccccccccccc}
 &
  \multicolumn{10}{c}{\textbf{Models}} \\
\multicolumn{1}{c|}{} &
  \multicolumn{2}{c||}{$2\_5$} &
  \multicolumn{2}{c||}{$2\_6$} &
  \multicolumn{2}{c||}{$2\_7$} &
  \multicolumn{2}{c||}{$3\_3$} &
  \multicolumn{2}{c|}{$4\_2$} \\ \hline
\multicolumn{1}{|c|}{\textbf{\begin{tabular}[c]{@{}c@{}}Result\end{tabular}}} &
  \multicolumn{2}{c||}{\texttt{SAT}} &
  \multicolumn{2}{c||}{\texttt{SAT}} &
  \multicolumn{2}{c||}{\texttt{SAT}} &
  \multicolumn{2}{c||}{\texttt{UNSAT}} &
  \multicolumn{2}{c|}{\texttt{UNSAT}} \\ \hline
\multicolumn{1}{|c|}{\textbf{Time (s)}} &
  \multicolumn{2}{c||}{8.2} &
  \multicolumn{2}{c||}{8.1} &
  \multicolumn{2}{c||}{8.12} &
  \multicolumn{2}{c||}{74.23} &
  \multicolumn{2}{c|}{85.1} \\ \hline
\end{tabular}

\caption{Single Check Verification on property $\phi_2$ of ACAS Xu.}
\label{tab:single_check}
\end{table}

\section{Full Experimental Results ACAS Xu}\label{fullACASExp}

We report in Tab\;\ref{tab:fullResultsACAS} the full results of comparing $\texttt{CountingProVe}$ and the exact counter on $\phi_2$ of the ACAS Xu benchmark discussed in Sec.\ref{resultsACAS}. As stated in the main paper, we consider only the model for which the property  $\phi_2$ does not hold (i.e., the models that present at least one single input configuration that violate the safety property). From the results of Tab\;\ref{tab:fullResultsACAS}, we can see that our approximation returns a tight interval confidence (mean of 2.83\%) of the VR for each model tested in about 4 hours. We want to underline that these obtained results do not exploit any particular optimization of our approximation, and therefore the times to compute these intervals can be greatly improved. For example, a simple optimization would compute the various $t$ iterations in parallel, significantly reducing computation times. Finally, Fig.\;\ref{fig:cprove} shows a 3d representation of the second property of the ACAS Xu benchmark (i.e., $\phi_2$), comparing the possible outcome of a standard formal verifier, namely DNN-Verification, and the problem presented in this paper \#DNN-Verification.

\begin{figure}[h!]
    \centering
    \includegraphics[width=0.85\linewidth]{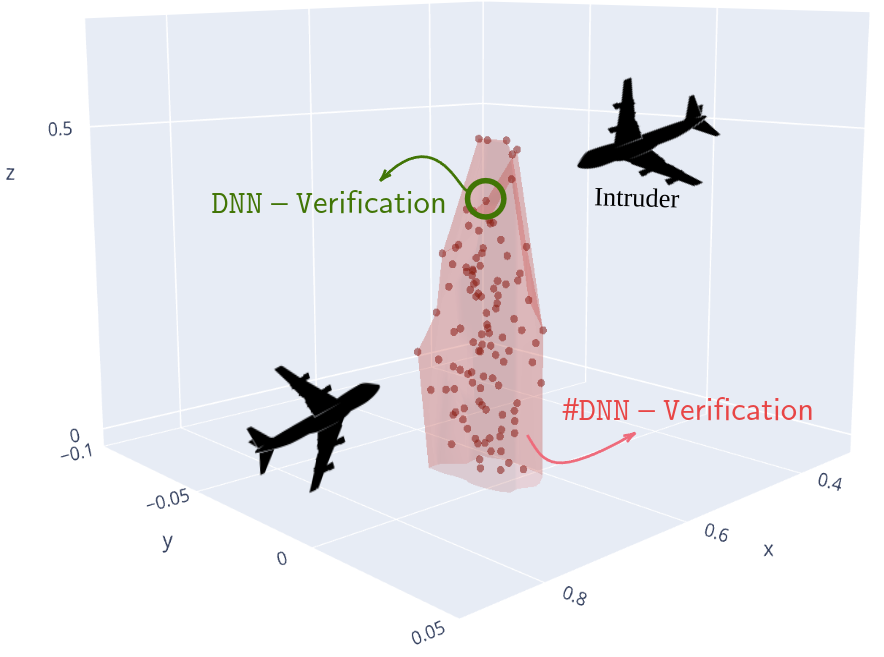}
    \caption{Explanatory image of the possible impact of \#DNN-Verification in safety-critical contexts. A standard verifier returns only a violation point (highlighted in the image with a green circle), limiting the interpretability of the results. In contrast, our approach paves the way to estimate the entire dangerous area (depicted in red in the figure).}
    \label{fig:cprove}
\end{figure}

\subsection*{Experiments on a Different Property}
To validate the correctness of our approximation, we also performed a final experiment on property $\phi_3$ of the Acas Xu benchmark. In particular, this property encodes a scenario in which if the intruder is directly ahead and is moving towards the ownship, the score for COC will not be minimal (we refer to \cite{Reluplex} for further details). This property is particularly interesting as it holds for all 45 models, i.e., we expect a 0\% of violation rate for each model tested. For simplicity, in Tab. \ref{tab:ACAS_3} we report only the first 5 models since the results were very similar for all the DNNs tested. As expected, we empirically confirmed that also for this particular situation, the lower and upper bounds computed with \texttt{CountingProVe} never overestimate and underestimate respectively the true value of the violation rate, in this case, 0\%.

\begin{table}[H]

\resizebox{\linewidth}{!}{
    \begin{tabular}{|c||ccc|}
    \hline
    
    \multicolumn{1}{|c||}{\textbf{Instance}} &
      \multicolumn{3}{c|}{\textbf{\texttt{CountingProVe}} (confidence $\geq$ 99\%)} \\
    \multicolumn{1}{|c||}{}  & Interval confidence VR & Size   & Time    \\[0.2ex] \hline 
    \multicolumn{1}{|l||}{\vspace{-2.5mm}} & & & \\
    \multicolumn{1}{|c||}{$\phi_3$ ACAS Xu\_1.1} & {[}0\%, 2.26\%{]}   & 2.26\%  & 215 min  \\[0.5ex]
    \multicolumn{1}{|c||}{$\phi_3$ ACAS Xu\_1.2} & {[}0\%, 2.88\%{]}   & 2.88\%  & 216 min  \\[0.5ex]
    \multicolumn{1}{|c||}{$\phi_3$ ACAS Xu\_1.3} & {[}0\%, 2.30\%{]}   & 2.30\%  & 215 min  \\[0.5ex]
    \multicolumn{1}{|c||}{$\phi_3$ ACAS Xu\_1.4} & {[}0\%, 2.47\%{]}   & 2.47\%  & 218 min \\[0.5ex]
    \multicolumn{1}{|c||}{$\phi_3$ ACAS Xu\_1.5} & {[}0\%, 2.48\%{]}   & 2.48\%  & 214 min \\\hline
    \end{tabular}
}
\caption{$\texttt{CountingProVe}$ on ACAS Xu $\phi_3$ property}
\label{tab:ACAS_3}
\end{table}

\begin{table*}

\tiny
\resizebox{\textwidth}{!}{
    \begin{tabular}{|cl||cc||ccc|}
    \hline
    
    \multicolumn{2}{|c||}{\textbf{Instance}} &
      \multicolumn{2}{c||}{\textbf{Exact Count}} &
      \multicolumn{3}{c|}{\textbf{$\texttt{CountingProVe}$} (confidence $\geq 99\%$)} \\
    \multicolumn{1}{|l}{} &
       &
      \multicolumn{2}{c||}{$\texttt{BaB}$} &
      \multicolumn{1}{l}{} &
      \multicolumn{1}{l}{} &
      \multicolumn{1}{l|}{} \\
    \multicolumn{1}{|l}{}      &       & Violation Rate & \multicolumn{1}{c||}{Time} & Interval confidence VR & Size   & Time    \\[0.2ex] \hline 
    \multicolumn{2}{|l||}{\vspace{-2.5mm}} & & \multicolumn{1}{c||}{} & & & \\
     \multicolumn{2}{|l||}{$\phi_2$ ACAS Xu\_2.1}  & --  & \multicolumn{1}{c||}{24 hrs}& {[}0.45\%, 5.01\%{]}   & 4.56\%  & 246 min \\[0.5ex] 
    \multicolumn{2}{|l||}{$\phi_2$ ACAS Xu\_2.2}  & --  & \multicolumn{1}{c||}{24 hrs} & {[}1.06\%, 4.81\%{]}   & 3.75\%  & 246 min \\[0.5ex] 
    \multicolumn{2}{|l||}{$\phi_2$ ACAS Xu\_2.3}  & -- & \multicolumn{1}{c||}{24 hrs} & {[}1.23\%, 4.21\%{]}   & 2.98\%  & 241 min \\[0.5ex] 
    \multicolumn{2}{|l||}{$\phi_2$ ACAS Xu\_2.4}  & -- & \multicolumn{1}{c||}{24 hrs} & {[}0.74\%, 3.43\%{]}   & 2.68\%  & 243 min \\[0.5ex] 
    \multicolumn{2}{|l||}{$\phi_2$ ACAS Xu\_2.5}  & -- & \multicolumn{1}{c||}{24 hrs} & {[}1.67\%, 4.10\%{]}   & 2.42\%  & 240 min \\[0.5ex] 
    \multicolumn{2}{|l||}{$\phi_2$ ACAS Xu\_2.6}  & -- & \multicolumn{1}{c||}{24 hrs}  & {[}1.01\%, 3.59\%{]}   & 2.58\%  & 248 min \\[0.5ex] 
    \multicolumn{2}{|l||}{$\phi_2$ ACAS Xu\_2.7}  & -- & \multicolumn{1}{c||}{24 hrs} & {[}2.35\%, 5.22\%{]}   & 2.87\%  & 240 min \\[0.5ex] 
     \multicolumn{2}{|l||}{$\phi_2$ ACAS Xu\_2.8}  & -- & \multicolumn{1}{c||}{24 hrs} & {[}1.77\%, 4.68\%{]}   & 2.92\%  & 248 min \\[0.5ex] 
      \multicolumn{2}{|l||}{$\phi_2$ ACAS Xu\_2.9}  & -- & \multicolumn{1}{c||}{24 hrs} & {[}0.18\%, 2.77\%{]}   & 2.59\%  & 239 min \\[0.5ex] 
       \multicolumn{2}{|l||}{$\phi_2$ ACAS Xu\_3.1}  & -- & \multicolumn{1}{c||}{24 hrs} & {[}1.62\%, 4.98\%{]}   & 3.36\%  & 242 min \\[0.5ex] 
        \multicolumn{2}{|l||}{$\phi_2$ ACAS Xu\_3.2}  & -- & \multicolumn{1}{c||}{24 hrs}  & {[}0\%, 2.50\%{]}   & 2.5\%  & 243 min \\[0.5ex] 
         \multicolumn{2}{|l||}{$\phi_2$ ACAS Xu\_3.3}  & -- & \multicolumn{1}{c||}{24 hrs} & {[}0\%, 2.54\%{]}   & 2.54\%  & 245 min \\[0.5ex] 
          \multicolumn{2}{|l||}{$\phi_2$ ACAS Xu\_3.4}  & -- & \multicolumn{1}{c||}{24 hrs} & {[}0.26\%, 3.08\%{]}   & 2.82\%  & 244 min \\[0.5ex] 
          \multicolumn{2}{|l||}{$\phi_2$ ACAS Xu\_3.5}  & -- & \multicolumn{1}{c||}{24 hrs} & {[}0.92\%, 3.60\%{]}   & 2.68\%  & 244 min \\[0.5ex] 
          \multicolumn{2}{|l||}{$\phi_2$ ACAS Xu\_3.6}  & -- & \multicolumn{1}{c||}{24 hrs} & {[}1.71\%, 4.48\%{]}   & 2.77\%  & 251 min \\[0.5ex] 
          \multicolumn{2}{|l||}{$\phi_2$ ACAS Xu\_3.7}  & -- & \multicolumn{1}{c||}{24 hrs} & {[}0.14\%, 2.64\%{]}   & 2.49\%  & 213 min \\[0.5ex] 
          \multicolumn{2}{|l||}{$\phi_2$ ACAS Xu\_3.8}  & -- & \multicolumn{1}{c||}{24 hrs} & {[}0.75\%, 3.28\%{]}   & 2.54\%  & 216 min \\[0.5ex] 
          \multicolumn{2}{|l||}{$\phi_2$ ACAS Xu\_3.9}  & -- & \multicolumn{1}{c||}{24 hrs} & {[}2.11\%, 5.20\%{]}   & 3.09\%  & 242 min \\[0.5ex] 
          \multicolumn{2}{|l||}{$\phi_2$ ACAS Xu\_4.1}  & -- & \multicolumn{1}{c||}{24 hrs} & {[}0.33\%, 3.04\%{]}   & 2.71\%  & 246 min \\[0.5ex] 
          \multicolumn{2}{|l||}{$\phi_2$ ACAS Xu\_4.3}  & -- & \multicolumn{1}{c||}{24 hrs} & {[}1.3\%, 3.61\%{]}   & 2.31\%  & 243 min \\[0.5ex] 
          \multicolumn{2}{|l||}{$\phi_2$ ACAS Xu\_4.4}  & --  & \multicolumn{1}{c||}{24 hrs} & {[}0.79\%, 3.57\%{]}   & 2.79\%  & 247 min \\[0.5ex] 
          \multicolumn{2}{|l||}{$\phi_2$ ACAS Xu\_4.5}  & -- & \multicolumn{1}{c||}{24 hrs} & {[}0.71\%, 4.03\%{]}   & 3.33\%  & 240 min \\[0.5ex] 
          \multicolumn{2}{|l||}{$\phi_2$ ACAS Xu\_4.6}  & -- & \multicolumn{1}{c||}{24 hrs} & {[}1.65\%, 4.72\%{]}   & 3.08\%  & 244 min \\[0.5ex] 
          \multicolumn{2}{|l||}{$\phi_2$ ACAS Xu\_4.7}  & -- & \multicolumn{1}{c||}{24 hrs} & {[}1.67\%, 4.33\%{]}   & 2.66\%  & 248 min \\[0.5ex] 
          \multicolumn{2}{|l||}{$\phi_2$ ACAS Xu\_4.8}  & -- & \multicolumn{1}{c||}{24 hrs} & {[}1.68\%, 4.17\%{]}   & 2.49\%  & 241 min \\[0.5ex] 
          \multicolumn{2}{|l||}{$\phi_2$ ACAS Xu\_4.9}  & -- & \multicolumn{1}{c||}{24 hrs} & {[}0.10\%, 2.61\%{]}   & 2.51\%  & 247 min \\[0.5ex] 
          \multicolumn{2}{|l||}{$\phi_2$ ACAS Xu\_5.1}  & -- & \multicolumn{1}{c||}{24 hrs} & {[}1.06\%, 3.76\%{]}   & 2.7\%  & 240 min \\[0.5ex] 
          \multicolumn{2}{|l||}{$\phi_2$ ACAS Xu\_5.2}  & -- & \multicolumn{1}{c||}{24 hrs} & {[}0.86\%, 3.58\%{]}   & 2.72\%  & 248 min \\[0.5ex] 
          \multicolumn{2}{|l||}{$\phi_2$ ACAS Xu\_5.4}  & -- & \multicolumn{1}{c||}{24 hrs} &  {[}0.75\%, 3.25\%{]}   & 2.5\%  & 239 min \\[0.5ex] 
          \multicolumn{2}{|l||}{$\phi_2$ ACAS Xu\_5.5}  & -- & \multicolumn{1}{c||}{24 hrs}  & {[}1.66\%, 4.35\%{]}   & 2.68\%  & 247 min \\[0.5ex] 
          \multicolumn{2}{|l||}{$\phi_2$ ACAS Xu\_5.6}  & -- & \multicolumn{1}{c||}{24 hrs} & {[}1.81\%, 4.45\%{]}   & 2.64\%  & 240 min \\[0.5ex] 
          \multicolumn{2}{|l||}{$\phi_2$ ACAS Xu\_5.7}  & -- & \multicolumn{1}{c||}{24 hrs} & {[}1.75\%, 5.15\%{]}   & 3.40\%  & 246 min \\[0.5ex] 
          \multicolumn{2}{|l||}{$\phi_2$ ACAS Xu\_5.8}  & -- & \multicolumn{1}{c||}{24 hrs} & {[}1.96\%, 4.65\%{]}   & 2.70\%  & 241 min \\[0.5ex]
    \multicolumn{2}{|l||}{$\phi_2$ ACAS Xu\_5.9}  & -- & \multicolumn{1}{c||}{24 hrs} & {[}1.62\%, 4.40\%{]}   & 2.77\%  & 241 min \\ \hline
    \multicolumn{2}{|l}{}  &  & \multicolumn{1}{c}{} &  \textbf{Mean} & 2.83\%  & 242 min \\\hline
    \end{tabular}
}
\caption{Comparison of $\texttt{CountingProVe}$ and exact counters on different benchmark setups.}\label{tab:fullResultsACAS}
\end{table*}

\end{document}